\documentclass[]{elsarticle}
\usepackage[utf8]{inputenc}
\usepackage{hyperref}
\journal{Journal of Discrete Applied Mathematics}

\usepackage{microtype}
\usepackage{amsmath,amssymb,amsthm}
\newtheorem{theorem}{Theorem}

\usepackage{mathtools}
\usepackage{booktabs}
\usepackage{csquotes}
\usepackage{microtype}

\newcommand\calH{\mathcal{H}}
\newcommand\hH{\hat{\calH}}
\newcommand\VV{\mathfrak{V}}
\newcommand\hVV{\hat{\mathfrak{V}}}
\DeclareMathOperator{\precision}{prec}
\DeclareMathOperator{\recall}{recall}

\DeclareMathOperator{\Mod}{Mod}

\renewcommand{\epsilon}{\varepsilon}
\newcommand{\algname}[1]{\textsc{#1}}
\def\attrs{\Phi}
\usepackage[noend]{algorithmic}
\usepackage{algorithm}

\algsetup{indent=2em}
\begin{document}

\begin{frontmatter}

\title{Probably approximately correct learning\\ of Horn envelopes from queries}

% \address{Radarweg 29, Amsterdam}
% \fntext[myfootnote]{Since 1880.}

% \author[danieladdress]{Daniel Borchmann}
\author[]{Daniel Borchmann}
\ead{daniel@algebra20.de}

\author[tomaddress]{Tom Hanika\corref{mycorrespondingauthor}}
\cortext[mycorrespondingauthor]{Corresponding author}
\ead{tom.hanika@cs.uni-kassel.de}

\author[sergeiaddress]{Sergei Obiedkov}
\ead{sergei.obj@gmail.com}

%\address[danieladdress]{Chair of Automata Theory, Technische Universität
%Dresden, Germany}
\address[tomaddress]{Knowledge \& Data Engineering Group,
 University of Kassel, Germany}
\address[sergeiaddress]{National Research University,
Higher School of Economics, Moscow, Russia}

\begin{abstract}
  We propose an algorithm for learning the Horn envelope of an arbitrary domain
  using an expert, or an oracle, capable of answering certain types of queries
  about this domain. Attribute exploration from formal concept analysis is a
  procedure that solves this problem, but the number of queries it may ask is
  exponential in the size of the resulting Horn formula in the worst case. We
  recall a well-known polynomial-time algorithm for learning Horn formulas with
  membership and equivalence queries and modify it to obtain a polynomial-time
  probably approximately correct algorithm for learning the Horn envelope of an
  arbitrary domain.
\end{abstract}


\end{frontmatter}

%\linenumbers

\section{Introduction}

The learnability of concepts from oracle queries has received significant attention in learning theory.  The most common types of oracles investigated
in the literature are membership and equivalence oracles, and for these types of
oracles various results have been obtained showing learnability in
polynomial time.  One of the most prominent examples is the
fact that Horn formulas can be learnt in polynomial time with access to
membership and equivalence oracles~\citep{angluin1992learning}.

In the realm of formal concept analysis \citep{ganter1999formal}, a different learning method has been
established almost simultaneously with the standard query learning setting.  The
theory of formal concept analysis emerged as a subfield
of mathematical order theory, more precisely of lattice theory, and it
studies lattices as \emph{hierarchies of concepts}.  Since its emergence in the
early 1980s, it has evolved into a rich theory with a wide range of
applications.
An important technique of formal concept analysis is the \emph{attribute exploration}
algorithm.  This algorithm aims at
learning a Horn representation, also called a \emph{Horn envelope}, of the
knowledge of a \emph{domain expert}.  A Horn envelope of a theory is a Horn
formula whose set of models includes all the models of the theory and is as
specific as possible \citep{kavvadias1993horn}. Here, a domain expert is an
oracle that is able to answer questions of the form \enquote{Does $A$ imply $B$
  in your domain?}, where $A$ and $B$ are conjunctions of atomic propositions.  If $A \to B$ is
indeed true, the expert confirms this implication.  Otherwise, the expert gives
a \emph{counterexample}, i.e., a model $C$ of the domain containing $A$ but not
$B$.

A large number of variants of the classical attribute exploration algorithm have
been investigated, and a wide range of applications have been proposed and
examined \citep{ganter2016conceptual}.  In particular, it turned out that the notion of a domain expert is
well suited for practical applications.  However, in the worst case, attribute
exploration requires exponential time in the number of propositional variables
and the size of the resulting Horn formula. This is because it enumerates all
the models of the domain as a byproduct, and their number may be exponential in
the size of the Horn formula.  On the other hand, an
\emph{exact} computation of the Horn envelope of real-world domains is rarely
useful in practice, as special cases may lead to artificial Horn formulas.

The problem of exponentially many queries does not exist in the case of using
membership and equivalence queries~\citep{angluin1992learning}, but in this
algorithm the queries are asked with respect to the Horn envelope rather than
with respect to the actual domain we want to explore.  Therefore, in our setting,
this algorithm is applicable only to Horn domains (for which the Horn envelope
is the same as the domain theory).  But even in this case, equivalence queries
may be hard to answer because they require an oracle to provide a negative
counterexample, a description of something that does not exist in the domain.

In this work we want to bring together the best of both approaches: we want to
devise a learning algorithm that requires only polynomial time in the size of
the output and issues only polynomially many queries to a domain expert.  To
this end, we propose a \emph{probably approximately correct} (PAC) version of
attribute exploration that computes an approximation of the Horn envelope of the
domain theory using queries about the validity of Horn formulas, just as in
classical attribute exploration.  We investigate two notions of approximation of
Horn envelopes: one is based on the agreement of a large fraction of models,
akin to the one used by~\cite{kautz1995horn}.  A second, novel, and stronger
notion called \emph{$\varepsilon$-strong Horn approximation} is based on the
requirement of the involved closure operators to coincide on a large
fraction of subsets.  The latter makes it possible to avoid some very weak
approximations, as we shall discuss later.

We state the problem precisely in
Section~\ref{sec:preliminaries}.  We then recall the algorithm
from~\citep{angluin1992learning} in Section~\ref{sec:afp}.  It serves the basis for our
PAC algorithms presented in Section~\ref{sec:pac}. The basic version does
not need counterexamples: it only needs the oracle to confirm or reject proposed
Horn clauses. Taking counterexamples into account makes it possible to reduce
the number of queries.  We show the effectiveness of our approach by means of
example with real-world data in Section~\ref{sec:evaluation}.

\section{Preliminaries}
\label{sec:preliminaries}

A \emph{Horn clause} over a set of propositional variables $\attrs$ is a disjunction of variables from $\attrs$ and their negations (i.e., \emph{literals}) containing at most one unnegated variable (\emph{positive literal}). The negated variables form the \emph{body} of the Horn clause, whereas the unnegated variable is called the \emph{head} of the clause. A \emph{definite Horn clause} contains exactly one positive literal. A \emph{Horn sentence} or \emph{Horn formula} is a conjunction of Horn clauses.
A Horn sentence consisting of definite Horn clauses with the same body can equivalently be represented by an \emph{implication} $p_1 \land \dots \land p_n \to q_1 \land \dots \land q_m$, where $p_i, q_i \in \Phi$. If one of the clauses sharing the body is not definite, i.e., if it contains no positive literal, the corresponding sentence can be represented by an implication $p_1 \land \dots \land p_n \to \bot$, where $\bot \not\in \attrs$ is the propositional constant falsum.

We will predominantly use set notation for representing Horn clauses and sentences. In particular, we will use notation $A \to B$, where $A, B \subseteq \attrs$, to represent the implication
\[\bigwedge_{p \in B}((\bigwedge_{q \in A}q) \to p).\]
Here, $A$ will be referred to as the \emph{premise} and $B$ as the \emph{conclusion} of the implication $A \to B$.  Abusing notation, we identify $\bot$ with the set $\attrs \cup \{\bot\}$, which implies, e.g., that $A \subseteq \bot$ and, consequently, $A \cap \bot = A$ for any $A \subseteq \attrs$. A Horn sentence $\calH$ will be regarded as a set of implications, and $|\calH|$ will stand for the number of implications in $\calH$.

A \emph{variable assignment} $V$ is a function that maps every propositional variable in $\Phi$ to 1 (true) or 0 (false). Again, we will often identify a variable assignment with the set of variables that it maps to 1. An assignment $V$ is a \emph{model} of a Horn clause $h$ (notation $V \models h$) if $h$ evaluates to 1 under the assignment $V$ (with the standard semantics of logical connectives). $V$ is a model of a Horn sentence $\calH$ (notation: $V \models \calH$) if it is a model of every clause it contains. As a special case, it is easy to see that $V$ is a model of an implication $A \to B$ if $A \not\subseteq V$ or $B \subseteq V$. We denote by $\Mod\calH$ the set of all models of $\calH$.

Two Horn sentences are \emph{equivalent} if they have exactly the same sets of
models. A Horn sentence $\calH_1$ \emph{entails} a Horn sentence $\calH_2$ if
every model of $\calH_1$ is a model of $\calH_2$ (notation: $\calH_1 \models
\calH_2$). % A Horn sentence $\calH$ is \emph{nonredundant} if none of the clauses it contains is entailed by the others, i.e., if $\calH \setminus \{h\} \models h$ for none of $h \in \calH$.
It is well-known that the set of models of a Horn sentence is closed under intersection. This makes it possible to  define $\calH(V)$ as the unique minimal model of $\calH$ in which 1 is assigned to all variables in $V$ and as $\attrs \cup \{\bot\}$ if no model containing $V$ exists. It is not difficult to see that $\calH(\cdot)$ is the \emph{closure operator} (i.e., it is monotone, extensive, and idempotent) corresponding to the closure system of models of $\calH$. Of course, $\calH(V) = V$ precisely for the models of $\calH$; we will sometimes refer to these models as sets \emph{closed} with respect to $\calH(\cdot)$. Obviously, if $\calH_1$ is equivalent to $\calH_2$, then $\calH_1(V) = \calH_2(V)$ for all $V \subseteq \attrs$.

Furthermore, a set of variable assignments is a set of models of a Horn sentence if and only if it is closed under intersection. We will denote the closure of a set $\VV$ of variable assignments under intersection by $\hVV$.  We call a Horn sentence $\calH$ a \emph{Horn envelope} for a set of assignments $\VV$ if $\hVV$ is precisely the set of models of $\calH$; note that, in this case, $\hVV = \{V \subseteq \attrs \mid V = \calH(V)\}$.

A set of variable assignments may have several equivalent envelopes. Of special interest, are envelopes that are minimal in the number of implications. One particular minimal envelope is known from formal concept analysis \citep{ganter1999formal} under the name of the Duquenne--Guigues or canonical basis of implications \citep{guigues1986famille}, which we define next.
A variable assignment $V$ is called \emph{pseudo-closed} with respect to a closure operator $\calH(\cdot)$ if
\begin{enumerate}
	\item $V \neq \calH(V)$;
	\item $\calH(W) \subsetneq V$ for every pseudo-closed $W \subsetneq V$.
\end{enumerate}
Note that, according to this definition, every variable assignment minimal among those that are not closed is pseudo-closed.

The \emph{Duquenne--Guigues basis} or \emph{canonical basis} of a Horn sentence $\calH$ is the following Horn sentence:
\begin{equation}\label{eq:dg}
	\bigwedge\{P \to \calH(P) \mid P\textrm{ is pseudo-closed with respect to } \calH(\cdot)\}.
\end{equation}
If $\calH$ is a Horn envelope of $\VV$, we also say that (\ref{eq:dg}) is the Duquenne--Guigues basis of $\VV$.

The problem of learning Horn envelopes frequently occurs in various settings, in particular, in data analysis, where Horn sentences are often used to summarize interdependencies between attributes in data. In this context, the data is given by a set $\VV$ of variable assignments and the task is to find its Horn envelope, i.e., a basis of implications valid in the data. However, the size of the Horn envelope $\hH$ of $\VV$ can be exponential in the size of $\VV$ \citep{kautz1995horn}. From the computational perspective, one could hope for an algorithm that runs in polynomial total time \citep{johnson1988generating}, that is, an algorithm polynomial in the size of input and output, i.e., in $|\attrs|$, $|\VV|$, and $|\hH|$, but no such algorithm is known yet. For this reason, it may be useful to compute Horn envelopes approximately.

Let $\hH$ be a Horn envelope of $\VV$, i.e., $\Mod\hH = \hVV$. We call a Horn sentence $\calH$ an \emph{$\varepsilon$-Horn approximation} of $\VV$ if
\begin{equation}\label{eq:approx}
	\frac{|\Mod\calH \bigtriangleup \Mod\hH|}{2^{|\attrs|}} \leq \varepsilon,
\end{equation}
where $A \bigtriangleup B$ is the symmetric difference between sets $A$ and $B$. This is the notion of approximation used in \citep{kautz1995horn}, where a probabilistic algorithm to compute such an approximation from a set of models in total polynomial time is presented.
However, this notion of approximation may be too weak for practical purposes: achieving an $\varepsilon$-Horn approximation of $\VV$ is very easy if $\hVV$ is small relative to $2^{|\attrs|}$, which is often the case. Since many real-world datasets are sparse, the size of $\hVV$ is often exponentially smaller than $2^{|\attrs|}$. Then setting $\calH = \{\varnothing \to \bot\}$ results in $\Mod\calH = \varnothing$, and the error \[\frac{|\Mod\calH \bigtriangleup \Mod\hH|}{2^{|\attrs|}} = \frac{|\Mod\hH|}{2^{|\attrs|}}\] is exponentially small.

Therefore, we will also use a stronger notion of approximation introduced in \citep{babin2012models}. We call $\calH$ an \emph{$\varepsilon$-strong Horn approximation} of $\VV$ if
\begin{equation}\label{eq:s-approx}
\frac{|\{V \subseteq \attrs \mid \calH(V) \neq \hat{\calH}(V)\}|}{2^{|\attrs|}} \leq \varepsilon,
\end{equation}
where $\hH$ is a Horn envelope of $\VV$. It is easy to see that an
$\varepsilon$-strong Horn approximation of $\VV$ is always an $\varepsilon$-Horn
approximation of $\VV$, but the reverse is not true.

\section{Learning Horn Sentences with Equivalence and Membership Queries}
\label{sec:afp}
In this paper, we consider the problem of learning Horn approximations via queries. In the query learning framework, rather than learning from a training dataset, the learning algorithm has access to an oracle (or an expert), which it can address with certain predefined types of questions \citep{angluin1988queries}. Probably, the most typical are equivalence and membership queries. In a \emph{membership query}, the learner asks whether a certain instance is an example of the concept being learned. For the problem of learning Horn sentences, the membership query allows the learning algorithm to find out whether a particular variable assignment is a model of the target Horn sentence. An \emph{equivalence query} is parameterized with a hypothesis describing the concept being learned. If the hypothesis matches the concept, the answer is positive and learning may be terminated. Otherwise, the oracle must provide a counterexample covered by the hypothesis, but not by the target concept (\emph{negative counterexample}), or vice versa (\emph{positive counterexample}). In our case, the target concept and hypotheses are Horn sentences and a counterexample is a variable assignment satisfying exactly one of these two sentences.

An algorithm for learning Horn sentences with equivalence and membership queries is described in \citep{angluin1992learning}, where it is proved that it requires time polynomial in the number of variables, $n$, and the number of clauses, $m$, of the target Horn sentence; $O(mn)$ equivalence queries and $O(m^2n)$ membership queries are made in the process. In the version of the algorithm we present here, the algorithm maintains a hypothesis $\calH$ consisting of implications of the form $A \to B$, where $A \subseteq B \subseteq \attrs \cup {\bot}$. The algorithm starts with the empty hypothesis, which is compatible with every possible assignment, and proceeds until a positive answer is obtained from the equivalence query. If a negative example $X$ is received instead, the algorithm uses membership queries  to find an implication $A \to B$ in the current hypothesis $\calH$ such that $A \cap X \neq A$ is not a model of the target Horn sentence. If such an implication is found, the implication $A \to B$ is replaced by $A \cap X \to B$, which ensures that $X$ is no longer a model of $\calH$. When a positive counterexample $X$ is obtained from an equivalence query, every implication $A \to B$ of which $X$ is not a model is replaced by $A \to B \cap X$ (recall that we identify $\bot$ with $\attrs \cup {\bot}$). We give pseudocode in Algorithm \ref{algo:afp} and refer the reader to \citep{angluin1992learning} for further details.

\begin{algorithm}
	\caption{\algname{Horn1}($equivalence(\cdot), member(\cdot)$)}
	\label{algo:afp}
	\begin{algorithmic}[1]
		\REQUIRE An equivalence and a membership oracles for a Horn sentence $\calH_*$.
		\ENSURE The Duquenne--Guigues basis of $\calH_*$ (represented as a set of implications).
		%\STATE
		\STATE $\calH := \varnothing$
		\WHILE{$equivalent(\calH)$ returns a counterexample $X$}
			\IF[negative counterexample]{$X \models \calH$}
				\STATE $found :=$ \textbf{false}
				\FORALL{$A \to B \in \calH$}
					\STATE $C := A \cap X$
					\IF{$A \neq C$ \textbf{and not} $member(C)$}\label{line:afp-member}
						\STATE $\calH := \calH \setminus \{A \to B\}$
						\STATE $\calH := \calH \cup \{C \to B\}$
						\STATE $found :=$ \textbf{true}
						\STATE \textbf{exit for}
					\ENDIF
				\ENDFOR
				\IF{\textbf{not} $found$}
					\STATE $\calH := \calH \cup \{X \to \bot\}$
				\ENDIF
			\ELSE[positive counterexample]
				\FORALL{$A \to B \in \calH$ such that $X \not\models A \to B$}
					\STATE $\calH := \calH \setminus \{A \to B\}$
					\STATE $\calH := \calH \cup \{A \to B \cap X\}$\COMMENT{If $B = \bot$, assume that $B = \attrs \cup \{\bot\}$}
				\ENDFOR
			\ENDIF
		\ENDWHILE
	\end{algorithmic}
\end{algorithm}

In \cite{arias2011construction}, it is shown that Algorithm \ref{algo:afp} always produces the Duquenne--Guigues basis of the target Horn sentence no matter what examples are received from the equivalence queries.

However, this algorithm has limitations in terms of applications we have in mind. In what situations query-based learning can be useful? One scenario is when there is not enough data about the domain under consideration, but there are domain experts willing to share their knowledge about the domain. We can use queries to extract information from them. Another scenario is when there is a huge amount of data, more than can be handled by standard algorithms for mining dependencies, and this data is organized in a distributed database or is spread over the Internet; however, there are mechanisms for efficiently querying the data. Query-based learning may also be useful if we work with a mathematical domain, one with an infinite number of objects, and there are procedures that can automatically prove theorems about the domain or generate counterexamples from this domain to our hypotheses; such procedures can be used as oracles, and we only need to ask them the right questions.

Unfortunately, it is not easy to use Algorithm \ref{algo:afp} to learn valid implications in such situations. One problem is that the algorithm needs negative counterexamples. These counterexamples are not part of the domain, they are propositional combinations that never occur. It is unreasonable to expect from a human expert to be able to easily produce such  combinations. A computer program can search a database or the Internet for a positive counterexample to a hypothesis, but it is more difficult to find something that does not exist. It may not always be easy to construct a mathematical object that violates a certain conjecture, but it seems much more difficult to construct a description of a non-existing object that satisfies the conjecture.

There is a more fundamental problem with applying Algorithm \ref{algo:afp} in our setting: the oracles in Algorithm \ref{algo:afp} must answer queries relative to the Horn formula being learnt. In our case, we work with an arbitrary domain and want to compute its Horn envelope; we assume that the oracle answers queries relative to the domain and not to its Horn envelope. If our domain is not Horn, i.e., its set of models $\VV$ is not closed under intersection, then the set $\hVV$ of models of its Horn envelope is different from $\VV$. Therefore, we will not receive a positive answer to an equivalence query even if we compute the envelope precisely; instead, we will obtain a negative counterexample from $\hVV \setminus \VV$. A similar problem occurs with membership queries: to be able to use Algorithm \ref{algo:afp}, we need the oracle to answer membership queries relative to $\hVV$, rather than to $\VV$.

\section{Learning Horn Envelopes of Arbitrary Domains}
\label{sec:pac}

A solution is offered by formal concept analysis in the form of a procedure called \emph{attribute exploration} \citep{ganter1999formal}. Instead of membership and equivalence queries, it uses what we will call \emph{implication queries}, i.e., queries of the form \enquote{Is $\VV \models A \to B$ true?} for $A, B \subseteq \attrs$. The oracle, or domain expert, answers positively in case the entailment holds or provides a \emph{positive counterexample} $X \in \VV$ such that $X \not\models A \to B$. In terms of \citep{angluin1988queries}, implication queries are a special case of \emph{superset queries}: asking whether $\VV \models A \to B$ amounts to asking whether the set of models of $A \to B$ is a superset of $\VV$.

%The goal is now to find a Horn approximation of $\mathfrak{V}$ by querying a
%corresponding domain expert, and this is exactly what the \emph{attribute
%  exploration algorithm} can do.  To this end, the algorithm enumerates all
%pseudo-closed sets of the set $\mathfrak{V}'$ of \emph{currently known models}
%as returned by the expert.  Whenever a pseudo-closed set $P$ is found, the
%expert is queried whether $P \to \mathcal{H}_{\mathfrak{V}'}(P)$ is true.  If
%this implication is true, it is stored; otherwise, the counterexample provided
%by the domain expert is added to $\mathfrak{V}'$ and the algorithm repeats until
%no new pseudo-closed set is found.
%
The algorithm only asks about the validity of implications that do not follow from those already confirmed by the expert and that do not contradict examples provided by the expert. Upon termination of the algorithm, the set of confirmed implications is the canonical basis of
$\VV$. Moreover, the set $\VV'$ of all models returned by the
 expert can be considerably smaller than $\VV$, but it has the same Horn envelope $\hH$. The downside is that the number of queries may be exponential in $\hH$, since $\VV'$ must contain all models of $\hH$ that cannot be represented as the intersection of other models of $\hH$; these are called \emph{characteristic} models of $\hH$ and their number can be exponential in $|\hH|$ \citep{kautz1995horn}.
Also, while deciding what queries must be posed, the algorithm implicitly enumerates all models in $\VV$.  In particular, the
time between two queries to the domain expert can be exponential in $|\attrs|$.

In the following, we present a modification of Algorithm \ref{algo:afp} that simulates membership queries relative to $\hVV$ by implication queries relative to $\VV$, the same queries as those used in attribute exploration. It also replaces equivalence queries by a call to a stochastic procedure, which makes it possible to compute an $\varepsilon$-Horn approximation of $\VV$ with the desired probability $\delta$. We will then modify this algorithm so that it produces an $\varepsilon$-strong Horn approximation of $\VV$. The resulting algorithms can be considered as PAC versions of attribute exploration.

\subsection{Simulating Membership Queries}

Let $\hH$ be a Horn envelope of a set $\VV \subseteq 2^\attrs$. For computing $\hH$, we need the membership query be answered relative to $\hVV$. Such a query can be simulated by several implication queries relative to $\VV$. One well-known (see, e.g., \cite{arias2017learning}) method to do this is presented in Theorem \ref{thm:member}.

\begin{theorem}
\label{thm:member}
	Let $\attrs$ be a set of variables, $A \subsetneq \attrs$, and $\VV \subseteq 2^\attrs$ be an arbitrary set of variable assignments. Then $A \in \hVV$ if and only if $\VV \models A \to \{a\}$ for no $a \in \attrs \setminus A$.
\end{theorem}
\begin{proof}
	If $\VV \models A \to \{a\}$ for some $a \in \attrs \setminus A$, then every assignment from $\VV$ that includes $A$ as a subset must contain $a$. But then, since $a \not \in A$, the set $A$ is not in $\VV$ and it cannot be an intersection of  assignments from $\VV$; i.e., $A \not\in \hVV$.

	Conversely, if $\VV \models A \to \{a\}$ for no $a \in \attrs \setminus A$, then, for every $a \in \attrs \setminus A$, there is $B \in \VV$ such that $A \subseteq B$, but $a \not \in B$. Hence, $A$ is the intersection of all $B \in \VV$ such that $A \subseteq B$; i.e., $A \in \hVV$.
\end{proof}

Theorem \ref{thm:member} makes it possible to check membership in $\hVV$ using at most $|\attrs|$ implication queries for every proper subset of $\attrs$. To check if $A \in \hVV$ for $A = \attrs$, one query $A \to \bot$ is sufficient. Of course, a positive answer to such a query means that $A \not\in \hVV$ for any subset $A$ of $\attrs$. This reasoning leads to Algorithm \ref{algo:member}.

\begin{algorithm}
	\caption{\algname{IsMember}($A$, $is\_valid(\cdot)$)}
	\label{algo:member}
	\begin{algorithmic}[1]
		\REQUIRE A set $A \subseteq \attrs$ and an implication oracle $is\_valid(\cdot)$ for some $\VV \subseteq 2^\attrs$.
		\ENSURE $\TRUE$ if $A \in \hVV$ and $\FALSE$ otherwise.
		%\STATE
		\IF{$is\_valid(A \to \bot)$}
			\RETURN \FALSE
		\ENDIF
		\FORALL{$a \in \attrs \setminus A$}
			\IF{$is\_valid(A \to \{a\})$}
				\RETURN \FALSE
			\ENDIF
		\ENDFOR
		\RETURN \TRUE
	\end{algorithmic}
\end{algorithm}

Note that, in this simulation, we do not use counterexamples provided by the implication oracle. We will call implication queries that do not return counterexamples \emph{restricted}. Thus, a membership query relative to $\hVV$ can be simulated by a linear (in $|\attrs|$) number of restricted implication queries relative to $\VV$. Since essentially all the algorithm does is posing queries and every next query can be obtained from the previous one in constant time, it is straightforward that the time complexity of Algorithm \ref{algo:member} is $O(|\attrs|)$ (of course, not including the time the oracle might need to answer the queries).

\subsection{Simulating Equivalence Queries}
\label{sec:approx}

We replace every equivalence query by sampling a number of variable assignments and checking whether any of them is a positive or negative counterexample. This technique, proposed in \cite{angluin1988queries}, makes it possible to obtain a polynomial-time PAC algorithm from a polynomial-time exact learning algorithm that uses equivalence queries. A similar strategy is used in \citep{kautz1995horn} to obtain a PAC algorithm computing an $\varepsilon$-Horn approximation of an explicitly given set of models. In our case, the difference is that we use this technique to transform an exact algorithm for learning a Horn theory with the membership oracle w.r.t. this theory into an algorithm for learning the Horn envelope of an arbitrary theory with the implication oracle w.r.t. this arbitrary theory.

In our algorithm, we sample $\Big\lceil\frac{1}{\varepsilon} \cdot \big(i + \ln\frac{1}{\delta}\big)\Big\rceil$ variable assignments to simulate the $i$th equivalence query asked by the algorithm. For each generated assignment $X$, we check if $X$ satisfies our hypothesis $\calH$ and, using Algorithm \ref{algo:member}, if $X \in \hVV$. If the answers to these questions are different, then $X$ is a counterexample to $\calH$. If none of the generated assignments is a counterexample, the algorithm concludes that $\calH$ is an $\varepsilon$-approximation of $\VV$.
We present the sampling procedure in Algorithm \ref{algo:approx-equiv} and the procedure that computes an $\varepsilon$-Horn approximation in Algorithm \ref{algo:pac-ae}.

\begin{algorithm}
	\caption{\algname{IsApproximatelyEquivalent}($\calH$, $is\_valid(\cdot)$, $\varepsilon$, $\delta$, $i$)}
	\label{algo:approx-equiv}
	\begin{algorithmic}[1]
		\REQUIRE A Horn formula $\calH$ over a set of propositional variables $\attrs$, an implication oracle $is\_valid(\cdot)$ for some $\VV \subseteq 2^\attrs$, $0 < \varepsilon \leq 1$, $0 < \delta \leq 1$, and $i \in \mathbb{N}$.
		\ENSURE A counterexample to $\calH$ relative to $\hat{\VV}$ if found; \TRUE, otherwise.
		%\STATE
		\FOR{$j := 1$ \TO $\Big\lceil\frac{1}{\varepsilon} \cdot \big(i + \ln\frac{1}{\delta}\big)\Big\rceil$}
			\STATE generate $X \subseteq M$ uniformly at random
			\IF{$(X \models \calH) \not\equiv \algname{IsMember}(X, is\_valid(\cdot))$}
				\RETURN $X$
			\ENDIF
		\ENDFOR
		\RETURN \TRUE
	\end{algorithmic}
\end{algorithm}

\begin{algorithm}
	\caption{\algname{HornApproximation}($is\_valid(\cdot)$, $\varepsilon$, $\delta$)}
	\label{algo:pac-ae}
	\begin{algorithmic}[1]
		\REQUIRE An implication oracle $is\_valid(\cdot)$ for some $\VV \subseteq 2^\attrs$, $0 < \varepsilon \leq 1$, and $0 < \delta \leq 1$.
		\ENSURE A set of implications $\calH$ that, with probability at least $1 - \delta$, is an $\varepsilon$-Horn approximation of $\VV$.
		%\STATE
		\STATE $\calH := \varnothing$
		\STATE $i := 1$
		\WHILE{\algname{IsApproximatelyEquivalent}($\calH$, $is\_valid(\cdot)$, $\varepsilon$, $\delta$, $i$) returns counterexample $X$}
			\IF[negative counterexample]{$X \models \calH$}
				\STATE $found := \FALSE$
				\FORALL{$A \to B \in \calH$}
					\STATE $C := A \cap X$
					\IF{$A \neq C$ \AND \NOT $\algname{IsMember}(C)$}
						\STATE $\calH := \calH \setminus \{A \to B\}$
						\STATE $\calH := \calH \cup \{C \to B\}$\label{line:modified-impl}% \cup (A \setminus X)\}$
						\STATE $found := \TRUE$
						\STATE \textbf{exit for}
					\ENDIF
				\ENDFOR
				\IF{\NOT $found$}
					\STATE $\calH := \calH \cup \{X \to \bot\}$\label{line:new-impl}
				\ENDIF
			\ELSE[positive counterexample]
				\FORALL{$A \to B \in \calH$ such that $X \not\models A \to B$}
					\STATE $\calH := \calH \setminus \{A \to B\}$
					\STATE $\calH := \calH \cup \{A \to B \cap X\}$\COMMENT{If $B = \bot$, assume that $B = \attrs \cup \{\bot\}$}
				\ENDFOR
			\ENDIF
			\STATE $i := i + 1$
		\ENDWHILE
	\end{algorithmic}
\end{algorithm}

\begin{theorem}
Let $\VV \subseteq 2^\attrs$ be an arbitrary set of variable assignments and $\hH$ be its Horn envelope. Given a (restricted) implication oracle for $\VV$, $0 < \varepsilon \leq 1$, and $0 < \delta \leq 1$ as input, Algorithm \ref{algo:pac-ae} computes an implication set $\calH$ that, with probability at least $1 - \delta$, is an $\varepsilon$-Horn approximation of $\VV$.
This algorithm runs in time polynomial in $|\attrs|$, $|\hH|$, $1/\varepsilon$, and $1/\delta$.
\end{theorem}

\begin{proof}
As shown in \citep{angluin1992learning}, Algorithm \ref{algo:afp} requires a number of counterexamples polynomial in $|\attrs|$ and $|\hH|$ no matter what counterexamples it receives. Suppose that this number is at most $k$. Since the only difference between Algorithm \ref{algo:afp} and Algorithm \ref{algo:pac-ae} is how queries get answered, the upper bound $k$ on the number of counterexamples will work for Algorithm \ref{algo:pac-ae}, too. We will make sure that the probability $\delta_i$ of failing to find a counterexample for the $i$th equivalence query using Algorithm \ref{algo:approx-equiv} is at most $\delta/2^i$. Then the probability of failing to find a counterexample for any of at most $k$ equivalence queries is bounded above by
\[\frac{\delta}{2} + \Big(1 - \frac{\delta}{2}\Big)\Bigg(\frac{\delta}{4} + \Big(1 - \frac{\delta}{4}\Big)\Bigg(\frac{\delta}{8} + \Big(1 - \frac{\delta}{8}\Big)\Bigg(\dots\Bigg(\frac{\delta}{2^{k-1}} + \Big(1 - \frac{\delta}{2^{k-1}}\Big)\frac{\delta}{2^k}\Bigg)\dots\Bigg)\Bigg)\Bigg)\leq\]
\[\leq \frac{\delta}{2} + \frac{\delta}{4} + \frac{\delta}{8} + \dots + \frac{\delta}{2^k} < \delta.\]

Let us assume that, at some point of the algorithm,
\[\frac{|\Mod\calH \bigtriangleup \Mod\hH|}{2^{|\attrs|}} > \varepsilon.\]
If this is not the case, then $\calH$ is already an $\varepsilon$-approximation of $\VV$, and it is safe to terminate the algorithm.
Under this assumption, if we choose $X$ randomly, we have $X \in \Mod\calH \bigtriangleup \Mod\hH$ with probability of at least $\varepsilon$. We check if this is the case with Algorithm \ref{algo:member}. If $X \in \Mod\calH \bigtriangleup \Mod\hH$, we use it as a counterexample to the equivalence query and proceed as in Algorithm \ref{algo:afp}. Otherwise, we generate another $X$. We make at most $l$ attempts at generating $X$; if we do not obtain a counterexample, we output $\calH$ and terminate.

The probability that we fail to find a counterexample in $l$ trials is smaller than $\delta_i$ if
\begin{equation}
	l > \frac{1}{\varepsilon} \cdot \ln\frac{1}{\delta_i}.\label{eq:trials}
\end{equation}
Indeed, the probability of failure is less than $(1 - \varepsilon)^l$. For this to be less than $\delta_i$, we need
\[l > \log_{1-\varepsilon}\delta_i = \frac{\ln\delta_i}{\ln(1 - \varepsilon)} = \frac{\ln(1/\delta_i)}{-\ln(1 - \varepsilon)}.\]
Since $-\ln(1 - \varepsilon) > \varepsilon$, it suffices to choose any $l$ satisfying (\ref{eq:trials}) to make the probability of failure less than $\delta_i$. In particular, we can set
\[l = \Bigg\lceil\frac{1}{\varepsilon} \cdot \ln\frac{1}{\delta_i}\Bigg\rceil = \Bigg\lceil\frac{1}{\varepsilon} \cdot \ln\frac{2^i}{\delta}\Bigg\rceil \leq \Bigg\lceil\frac{1}{\varepsilon} \cdot \Big(i + \ln\frac{1}{\delta}\Big)\Bigg\rceil \leq \Bigg\lceil\frac{1}{\varepsilon} \cdot \Big(poly(|\attrs|, |\hH|) + \ln\frac{1}{\delta}\Big)\Bigg\rceil.\]

To sum up, Algorithm \ref{algo:afp} runs in time polynomial in $|\attrs|$ and the number of implications in the target Horn sentence $\hH$. We simulate this algorithm, but replace each equivalence query by a number of attempts polynomial in $|\attrs|$, $|\hH|$, $1/\varepsilon$, and $1/\delta$ at generating a counterexample to the current hypothesis $\calH$. Each such attempt requires time $poly(\attrs, |\hH|)$, in particular, since the algorithm guarantees that $|\calH| \leq |\hH|$. Therefore, our simulation runs in time polynomial in $|\attrs|$, $|\hH|$, $1/\varepsilon$, and $1/\delta$ and, as argued above, produces an $\varepsilon$-approximation of $\VV$ with probability at least $1 - \delta$.
\end{proof}
We are well aware that this result, in another form, is known from literature
\citep{angluin1992learning}. However, we included this result on the one hand to
show that it also holds with the new form of implication oracle, and on the
other hand to include all details in order to present an comprehensive
exposition.

%\section{Using Closure Operator for More Efficient Search of Counterexamples}

\subsection{Strong Approximations}

The algorithm we have just presented can be modified to compute $\varepsilon$-strong Horn approximations. We only need to modify the way counterexamples are generated by the \algname{IsApproximatelyEquivalent} procedure.

If
\begin{equation}\label{eq:strong-approx}
	\frac{|\{V \subseteq \attrs \mid \calH(V) \neq \hat{\calH}(V)\}|}{2^{|\attrs|}} > \varepsilon,
\end{equation}
then, by generating $X$ uniformly at random, we obtain $X$ such that $\calH(X) \neq \hat{\calH}(X)$ with probability at least $\varepsilon$. Suppose that we have generated such an $X$. The problem is that this $X$ is not necessarily a counterexample in the sense required by the algorithm, because it may happen that it belongs neither to $\Mod\calH$ nor to $\hVV$. It turns out that we can use $X$ to manufacture a counterexample in time polynomial in $|\attrs|$.

\begin{theorem}
\label{thm:closure}
	Let $\hH$ be the Horn envelope of $\VV \subseteq 2^{\attrs}$ and $\calH$ be a Horn formula over $\attrs$. Then $\calH(X) = \hH(X)$ if and only if $\calH(X) \in \hVV \cup \{\bot\}$ and $\VV \models X \to \calH(X)$.
\end{theorem}
\begin{proof}
	Suppose that $\calH(X) = \hH(X) \neq \bot$. Then $\hH(X) \in \hVV$ and $\VV \models X \to \hH(X)$, and we also have $\calH(X) \in \hVV$ and $\VV \models X \to \calH(X)$. If, on the other hand, $\calH(X) = \hH(X) = \bot$, then $X$ is a subset of no model in $\VV$ and $\VV \models X \to \bot$.

	Conversely, if $\VV \models X \to \calH(X)$, then $\calH(X) \subseteq \hH(X)$; and, if $\calH(X) \in \hVV$, then $\hH(X)$, the minimal superset of $X$ from $\hVV$, must be a subset of $\calH(X)$, i.e., $\hH(X) \subseteq \calH(X)$. The latter must also hold if $\calH(X) = \bot$.
\end{proof}

To obtain a counterexample from a randomly generated $X$, we first compute $\calH(X)$ and query the oracle to verify the implication $X \to \calH(X)$. If the implication is invalid, the oracle will return a positive counterexample $C$. Otherwise, we check if $\calH(X) \in \hVV$ using the \algname{IsMember} procedure. If the outcome is negative, then $\calH(X)$ is a negative counterexample; else, from Theorem \ref{thm:closure}, we know that $\calH(X) = \hH(X)$, and we generate another $X$ unless we have reached the maximum number of iterations. Algorithm \ref{algo:strong-approx-equiv} gives the pseudocode.

Thus, given (\ref{eq:strong-approx}), the probability of finding a counterexample at one iteration of Algorithm \ref{algo:strong-approx-equiv} is greater than $\varepsilon$, and the same reasoning as in Section \ref{sec:approx} leads to the following theorem.

\begin{theorem}
Let $\VV \subseteq 2^\attrs$ be an arbitrary set of variable assignments and $\hH$ be its Horn envelope. Given an implication oracle for $\VV$, $0 < \varepsilon \leq 1$, and $0 < \delta \leq 1$ as input and using Algorithm \ref{algo:strong-approx-equiv} as the \algname{IsApproximatelyEquivalent} procedure, Algorithm \ref{algo:pac-ae} computes an implication set $\calH$ that, with probability at least $1 - \delta$, is an $\varepsilon$-strong Horn approximation of $\VV$.
This algorithm runs in time polynomial in $|\attrs|$, $|\hH|$, $1/\varepsilon$, and $1/\delta$.
\end{theorem}

\begin{algorithm}
	\caption{\algname{IsStronglyApproximatelyEquivalent}($\calH$, $is\_valid(\cdot)$, $\varepsilon$, $\delta$, $i$)}
	\label{algo:strong-approx-equiv}
	\begin{algorithmic}[1]
		\REQUIRE A Horn formula $\calH$ over a set of propositional variables $\attrs$, an implication oracle $is\_valid(\cdot)$ for some $\VV \subseteq 2^\attrs$, $0 < \varepsilon \leq 1$, $0 < \delta \leq 1$, and $i \in \mathbb{N}$.
%                	\renewcommand{\algorithmicrequire}{\textbf{Interactive input:}}
%		\REQUIRE Upon request, confirm that $\VV \models A \to B$ for $A, B \subseteq \attrs$ and the target set $\VV$ of models or give a counterexample $X \in \VV$ such that $X \not\models A \to B$.
		\ENSURE A counterexample to $\calH$ with respect to $\hat{\VV}$ if found; \TRUE, otherwise.
		%\STATE
		\FOR{$j := 1$ \TO $\Big\lceil\frac{1}{\varepsilon} \cdot \big(i + \ln\frac{1}{\delta}\big)\Big\rceil$}
			\STATE generate $X \subseteq M$ uniformly at random
			\STATE $Y := \calH(X)$
			\IF{$is\_valid(X \to Y)$ returns a counterexample $C$}
				\RETURN $C$\COMMENT{$C$ is a positive counterexample}
			\ENDIF
			\IF{\NOT \algname{IsMember($Y$, $is\_valid(\cdot)$)}}
				\RETURN $Y$\COMMENT{$Y$ is a negative counterexample}
			\ENDIF
		\ENDFOR
		\RETURN \TRUE
	\end{algorithmic}
\end{algorithm}

%\begin{algorithm}
%	\caption{\algname{IsApproximatelyEquivalent}($\calH$, $\varepsilon$, $\delta$, $i$)}
%	\label{algo:cc}
%	\begin{algorithmic}[1]
%		\REQUIRE %A Horn formula $\calH$ set of propositional variables $\attrs$ and a teacher $t$ of a certain concept expressed as a Horn formula over $\attrs$.
%		\ENSURE %A set of Horn clauses.
%		%\STATE
%		\FOR{$j := 1$ \TO $\Big\lceil\frac{1}{\varepsilon} \cdot \big(i + \ln\frac{1}{\delta}\big)\Big\rceil$}
%			\STATE generate $X \subseteq M$ unformly at random
%			\STATE $Y := \calH(X)$
%			\IF[$X$ is a positive counterexample]{\algname{IsMember}($X$) \AND $X \neq Y$}
%				\RETURN $X$
%			\ENDIF
%			\IF[$Y$ is a negative counterexample]{\NOT \algname{IsMember($Y$)}}
%				\RETURN $Y$
%			\ENDIF
%		\ENDFOR
%		\RETURN \TRUE
%	\end{algorithmic}
%\end{algorithm}

%% from our review work:
Summing this subsection up, strong approximation copes with the problem of having
generated a set that may not be a counterexample. Using this to obtain a valid
counterexample in an efficient way, and, in our opinion, a novelty.

\subsection{Variations and Optimizations}

The algorithm can be modified so that its current hypothesis $\calH$ is always such that $\VV \models \calH$. To ensure this, we need to take some care when adding implications in lines \ref{line:modified-impl} and \ref{line:new-impl} of Algorithm \ref{algo:pac-ae}. For example, instead of adding implication $X \to \bot$, we should check via an implication query whether it is valid, and, if not, add instead implication $X \to \hH(X)$ by computing $\hH(X)$, again, using implication queries. One way to do this is to query about the validity of implications of the form $X \to \{a\}$ for $a \in \attrs \setminus X$: those $a$ for which the answer is positive belong to $\hH(X)$.
With this modification, our sampling procedure that replaces the equivalence oracle will return only negative counterexamples, and thus the part of Algorithm \ref{algo:pac-ae} dealing with positive counterexamples can be eliminated.

To reduce the number of queries, we can cache counterexamples returned by the oracle. All these counterexamples are models from $\VV$, and thus they can be used to falsify some implications without resorting to the oracle: if an implication $A \to B$ has a counterexample among the models obtained so far, a query about its validity is not necessary. Since the total number of queries submitted by the algorithm is polynomial in all the quantities we care about, so is the number of counterexamples received from the oracle, and, consequently, the memory and time overhead incurred by this modification is also polynomial.

Similarly, we can cache the implications confirmed by the oracle and use them to verify the validity of some implications. It is also worth exploring whether integrating such confirmed implications into the current hypothesis may be useful.

\section{Experimental Evaluation}
\label{sec:evaluation}

Our algorithms come with a theoretical guarantee on the quality of approximation or, to be more precise, on the probability of attaining the desired quality. In Section \ref{sec:measures}, we suggest quality measures \emph{precision} and \emph{recall}, which are slightly different from those of (\ref{eq:approx}) and (\ref{eq:s-approx}) for which the algorithms were designed. In Section~\ref{sec:data-sets-results}, we experimentally evaluate the quality of approximations computed by Algorithm \ref{algo:pac-ae} in terms of these measures.

In general, the domain expert, or the oracle, used in learning is not necessarily a human being: it may well be a knowledge base equipped with a procedure capable of answering implication queries.
To easily obtain domain experts for
our experiments, we make use of the following approach. Starting from a data
set $\mathfrak{V}$, we simulate a domain expert for $\mathfrak{V}$ by confirming
$A \to B$ if $\mathfrak{V} \models A \to B$. Otherwise, the expert returns a
counterexample to $A \to B$ from the dataset. The datasets we use are described in Section~\ref{sec:data-sets}.

\subsection{Precision and Recall}\label{sec:measures}
Informally, precision measures how often the extracted implications infer only correct knowledge from a given variable assignment. Conversely, recall measures how often the knowledge inferred from a variable assignment is complete.

%\begin{definition}
%  \label{def:compprec}
More formally, let $\attrs$ be a finite set, let $\mathfrak{V}$ be a set of variable assignments over $\attrs$, $\hH$ be its Horn envelope, and $\mathcal{H}$ be a set of implications.  Then the \emph{precision} and \emph{recall} of $\mathcal{H}$ with respect to $\mathfrak{V}$ are defined by
  \begin{align*}
    \precision_{\mathfrak{V}}(\mathcal{H}) &\coloneqq \frac{\lvert\{A\subseteq\attrs\mid \mathfrak{V}\models A\to \mathcal{H}(A)\}\rvert}{2^{\lvert\attrs\rvert}},\\
    \recall_{\mathfrak{V}}(\mathcal{H}) &\coloneqq \frac{\lvert\{A\subseteq\attrs\mid \mathcal{H}\models A\to \hH(A)\}\rvert}{2^{\lvert\attrs\rvert}}.
  \end{align*}
%\end{definition}

One can see that precision and recall are, in a way, two sides of strong approximation as defined by (\ref{eq:s-approx}).

Computing the exact values of precision and recall for sufficiently large sets
$\attrs$ is infeasible and, for our experimental evaluation, is not necessary: a good approximation of the values would be enough.  To obtain such approximations, we sample a certain number of subsets $A \subseteq \attrs$ and count how often the corresponding condition is true.  More precisely, to obtain a good approximation of $\precision_{\mathfrak{V}}(\mathcal{H})$ and $\recall_{\mathfrak{V}}(\mathcal{H})$, we randomly choose a subset $\mathcal{T} \subseteq 2^\attrs$ and compute
\begin{align*}
  \precision^{\approx}_{\mathfrak{V}}(\mathcal{H})
  &\coloneqq \frac{\lvert\{A\in\mathcal{T}\mid \mathfrak{V} \models A\to \mathcal{H}(A)\}\rvert}{\lvert\mathcal{T}\rvert},\\
  \recall^{\approx}_{\mathfrak{V}}(\mathcal{H})
  &\coloneqq \frac{\lvert\{A\in\mathcal{T}\mid \mathcal{H}\models A\to \hH(A)\}\rvert}{\lvert\mathcal{T}\rvert},
\end{align*}

An immediate question is what size $n$ the sample set $\mathcal{T}$ needs to have
for the approximation to be a good one.  Utilizing Hoeffding's
inequality~\citep{hoeffding1963}, we obtain for fixed $0 < \eta, t$ that
\begin{align*}
  \Pr(\precision_{\mathfrak{V}}(\mathcal{H}) - \precision^{\approx}_{\mathfrak{V}}(\mathcal{H}) \geq t)
  &< \eta,\\
  \Pr(\recall_{\mathfrak{V}}(\mathcal{H}) - \recall^{\approx}_{\mathfrak{V}}(\mathcal{H}) \geq t)
  &< \eta
\end{align*}
for
\begin{equation*}
  n \geq \frac{1}{2t^{2}}\cdot\ln\frac{1}{\eta}.
\end{equation*}
For our experiments, we chose $\eta = 0.001$ and $t = 0.01$, resulting in $n \approx 35000$ samples.

\subsection{Datasets}
\label{sec:data-sets}

We utilized various datasets with various properties. All used
datasets were obtained from the UCI Machine Learning
Repository~\citep{Lichman:2013}. The particular choice for the Zoo
dataset and the Breast Cancer dataset was made due to
the fact that those datasets are almost Boolean, well investigated, and of
moderate size, thus suiting our experiments. For comparison reasons,
we also considered randomly generated datasets that were of the same
size and density as the ones we use from the UCI Machine Learning
Repository.

\paragraph{Zoo Dataset (ZD)}
\label{sec:zoo-data-set}

This dataset, created by Richard Forsyth, consists of 101 animals described by
15 attributes. From these attributes, 14 are Boolean and have been used as they
are. Examples include attributes \emph{(has) feathers}, \emph{(is) airborne}, and \emph{(has a)
  backbone}. The two remaining attributes \emph{(number of) legs} and
\emph{type} were replaced by \emph{legs = 0}, \emph{legs = 2}, \emph{legs = 4},
\emph{legs = 5}, \emph{legs = 6}, \emph{legs = 8}, \emph{type = 1}, \dots,
\emph{type = 6}.  The models of this dataset are then the combinations of
attributes occurring in it.

\paragraph{Breast Cancer Dataset (BC)}
\label{sec:breast-cancer-data}

This dataset was originally obtained from the University of Wisconsin
Hospitals, Madison from Dr. William H. Wolberg~\citep{wdbc}. It
consists of 699 named instances, each representing a clinical case described by
nine numeric attributes such as \emph{Uniformity of Cell Size}, \emph{Bare
  Nuclei}, and \emph{Marginal Adhesion}.  Each of these attributes can have a
value between one and ten, and these attributes were turned into Boolean
attributes in the same way as for the ZD dataset.  Finally, one attribute classifies a clinical
case as malignant or benign.  The models of this dataset are again the
combinations of attributes occurring in it.

\paragraph{Random Dataset (RD)}

For both the Zoo dataset and the Breast Cancer dataset, we generated ten random
datasets, all with the same number of attributes, instances, as well as
incidence probability.  These datasets have been obtained by randomly choosing
whether an instance possesses an attribute, with the same probability as for the
original datasets.  Note that while the process places incidences uniformly at
random, the Horn envelopes of the resulting set of models do not have to be
distributed uniformly, as discussed in~\cite{conf/cla/BorchmannH16}.

\subsection{Experimental Results}
\label{sec:data-sets-results}

For the various datasets described above, we conducted two types of experiments.
Firstly, we ran Algorithm~\ref{algo:pac-ae} for
various choices of $\varepsilon$ and $\delta$ and computed the precision, recall,
fraction of valid implications, as well as the number of computed implications.
The purpose of these experiments is to investigate the quality of the approximation
returned by the algorithm.  Secondly, we repeated the
algorithm a certain number of times and investigated the
distribution of precision, recall, fraction of correct implications, as well as
the number of implications.  The purpose here is to see how much the results can vary between runs of the algorithm.

%Since our implementation is not yet optimized,
%we refrain from providing specific running times.
%However, we shall give rough estimates whenever possible.

\paragraph{Single Runs of \algname{HornApproximation}}

We begin our discussion with the results for the Zoo dataset. We ran
Algorithm \ref{algo:pac-ae} varying $\varepsilon$ from $\{0.01,0.1,0.5\}$ and
$\delta$ from $\{0.1,0.9\}$, three times each. We chose those particular values for
$\varepsilon$ and $\delta$ such that our estimates of precision and recall differ from the
true values by at most 1\%, 10\%, and 50\% with high as well as low
probability. Increasing $\varepsilon$ further seems unreasonable for
real world applications.
% Each run for this experiment took
%around seconds on a recent laptop.

We observed different outcomes for different parameter combinations, as shown in Table~\ref{tab:res}. Among the
computed implications were several combining different attributes, e.g.
\begin{equation*}
  \{\emph{airborne}, \emph{breathes}, \emph{venomous}\} \to \{\emph{eggs},
  \emph{type=6}, \emph{leg=6}, \emph{hair}\}.
\end{equation*}
A complete list of implications for one run is shown in
the end of this section.   The precision was always 1, and was
therefore not included in Table~\ref{tab:res}.  The recall is very volatile in
our experiments.
Varying the $\varepsilon$ parameter has a big impact on the size of the
resulting set of implications: the smaller $\varepsilon$, the more
implications are found.  The increase of the number of learned implications when
$\varepsilon$ is decreased is expected, since with more samples more queries to
the oracles can be stated.  Indeed, choosing $\varepsilon= 1/100,1/1000,1/10000$
resulted in bases of sizes 24, 38, and 95, respectively.  Note that the
Duquenne-Guigues basis of ZD has 141 implications.  On the other hand, more
queries do not necessarily lead to more implications, as shown by the results in
Table~\ref{tab:res}.
%Running Algorithm \ref{algo:pac-ae} with $\varepsilon = 1/10000$ took about four hours.
We also
counted the number of queries to the expert, which were 59852, 1016796, and
53455186, for the three values of $\varepsilon$ respectively.

%The high values of recall and precision for smaller sets of
%implications should be taken into account carefully. The implications for small
%values of $\varepsilon$ degrade to such with premises of size one, which are not
%that useful anyway. Finally, as already observed by~\cite{BHO2017}, the
%influence of $\delta$ is small or negligible, as suggested by the number of
%randomly sampled counterexamples used when simulating equivalence queries.

The BC dataset %, the observations above hold true as well. In general, the
%computation took longer, about five minutes, since this dataset
has six times
as many attributes as the Zoo dataset. Its Duquenne--Guigues basis
consists of 10739 implications. Compared with the Zoo dataset, an inferior
recall for higher values of $\varepsilon$ can be observed.  However, the
precision, as well as the fraction of correctly computed implications, do not
seem to be correlated with $\varepsilon$.

\begin{table}[t]
  \centering
  \begin{tabular}{l|ccc|ccc|ccc}
    \toprule
    Name&SR$_{1}$&DP$_1$&BS$_{1}$&SR$_{2}$&DP$_2$&BS$_{1}$&SR$_{3}$&DP$_3$&BS$_{3}$\\
    \midrule
    ZD$_{(0.01,0.1)}$&\textbf{0.91}&0.75&24&0.89&0.87&23&0.88&\textbf{0.96}&\textbf{26}\\
    ZD$_{(0.01,0.9)}$&0.08&0.71&24&\textbf{0.90}&\textbf{0.92}&\textbf{28}&0.81&0.74&26\\
    ZD$_{(0.1,0.1)}$&0.09&\textbf{1.00}&\textbf{17}&\textbf{0.24}&0.79&14&0.00&0.75&14\\
    ZD$_{(0.1,0.9)}$&0.19&0.73&11&\textbf{0.75}&0.73&11&0.49&\textbf{0.73}&\textbf{15}\\
    ZD$_{(0.5,0.1)}$&0.07&1.00&10&\textbf{0.37}&1.00&11&0.00&\textbf{1.00}&\textbf{11}\\
    ZD$_{(0.5,0.9)}$&0.73&0.89&9&0.54&0.78&9&\textbf{0.73}&\textbf{1.00}&\textbf{11}\\
    \midrule
    BC$_{(0.01,0.1)}$&1.00&0.95&39&0.99&\textbf{0.97}&38&\textbf{1.00}&0.96&\textbf{50}\\
    BC$_{(0.01,0.9)}$&\textbf{1.00}&0.95&41&1.00&0.94&\textbf{47}&1.00&\textbf{0.98}&44\\
    BC$_{(0.1,0.1)}$&\textbf{0.99}&\textbf{0.97}&\textbf{31}&0.93&0.96&26&0.98&0.93&29\\
    BC$_{(0.1,0.9)}$&0.88&0.94&33&0.97&0.90&29&\textbf{0.99}&\textbf{0.97}&\textbf{35}\\
    BC$_{(0.5,0.1)}$&0.84&1.00&22&\textbf{0.88}&\textbf{1.00}&\textbf{24}&0.67&1.00&21\\
    BC$_{(0.5,0.9)}$&0.75&\textbf{1.00}&25&\textbf{0.91}&1.00&24&0.79&0.93&\textbf{28}\\
  \end{tabular}
  \caption{Results for the Zoo (ZD) and Breast Cancer (BC) experiments for all
    parameter combinations and all three runs. SR = the recall, DP = the fraction of
    valid implications, BS = the number of computed implications.}
  \label{tab:res}
\end{table}

Finally, for each random dataset we applied our algorithm and calculated the
average value and the standard deviation of the size of the set of
implications, the fraction of correctly computed implications, and the recall.
We used $\varepsilon=0.1$ and $\delta=0.1$.  For
the Zoo dataset, we obtained around $23.1\pm3.8$ implications, with a fraction
of $0.84\pm0.12$ valid ones, and recall around $0.90\pm0.05$.  For the
Breast Cancer dataset, we obtained $24\pm1.3$ implications, $0.94\pm0.04$ of
which were valid, and a recall of $0.97\pm0.01$.

The size of the set of implications dropped for the Breast Cancer dataset
significantly, from about 30 to approximately 24.  On the contrary, we see an
increase from around 15 to 24 in the Zoo dataset. For both datasets, we can
observe that the fraction of valid implications is about the same
in the random dataset and the Zoo and Breast Cancer datasets, respectively.
However, the recall in the Breast Cancer case stays the same, whereas in the Zoo
case the recall for the random dataset is considerable larger than for the
original dataset. The standard deviation for both measures is considerably
small.  We conjecture that the drop in the number of implications obtained for the
Breast Cancer dataset might be attributed to the random generation process:
while generating the random datasets, we did not take into account that multiple
values of a numeric attribute should still exclude each other. Since Breast
Cancer dataset contains many numeric attributes, this effect could be large.

\paragraph{Repeated Runs of \algname{HornApproximation}}
\label{sec:stability}

How reliable is the computation for a particular set of parameters? Since the
results in the previous section revealed a high volatility, especially for the
recall measure, we wanted to check how reliable the results of the algorithm were in terms of reproducibility. For this, we applied
the algorithm 1000 times to the Zoo dataset using
$\varepsilon\in\{0.01,0.1,0.5,0.9\}$ with $\delta=0.1$. The results are shown in
Figure~\ref{fig:stab}.

\begin{figure}[t]
  \centering
  \includegraphics[scale=0.3]{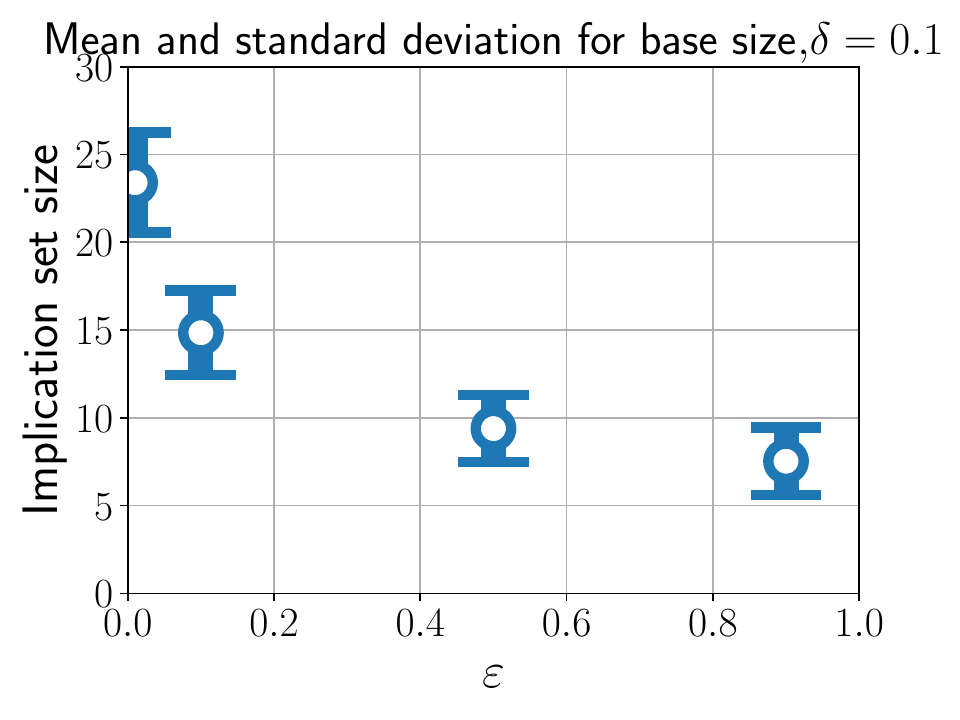}
  \includegraphics[scale=0.3]{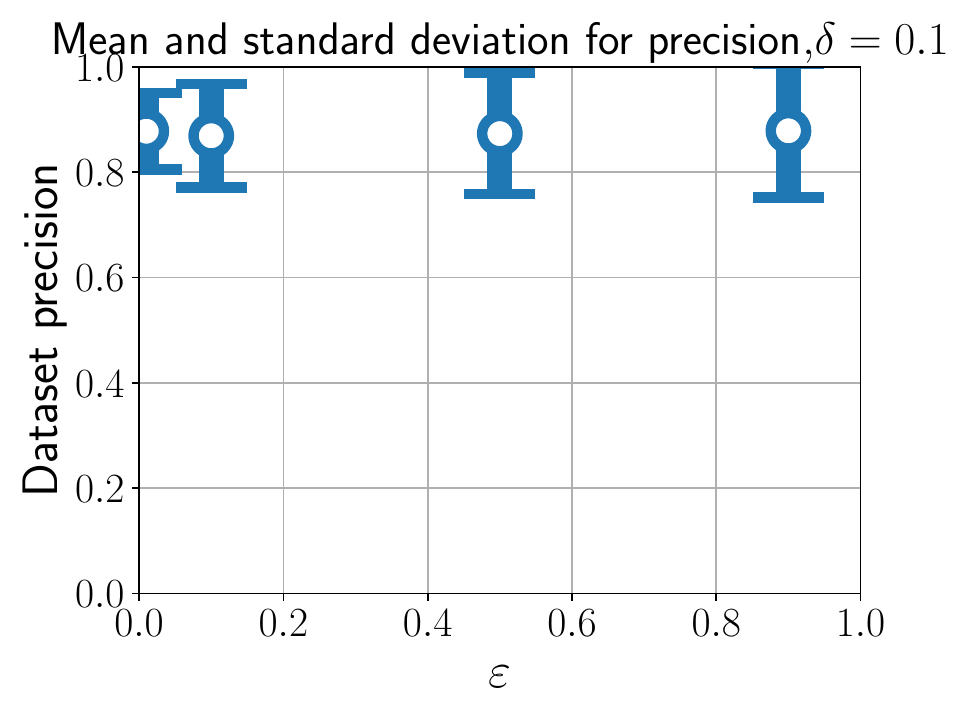}
  \includegraphics[scale=0.3]{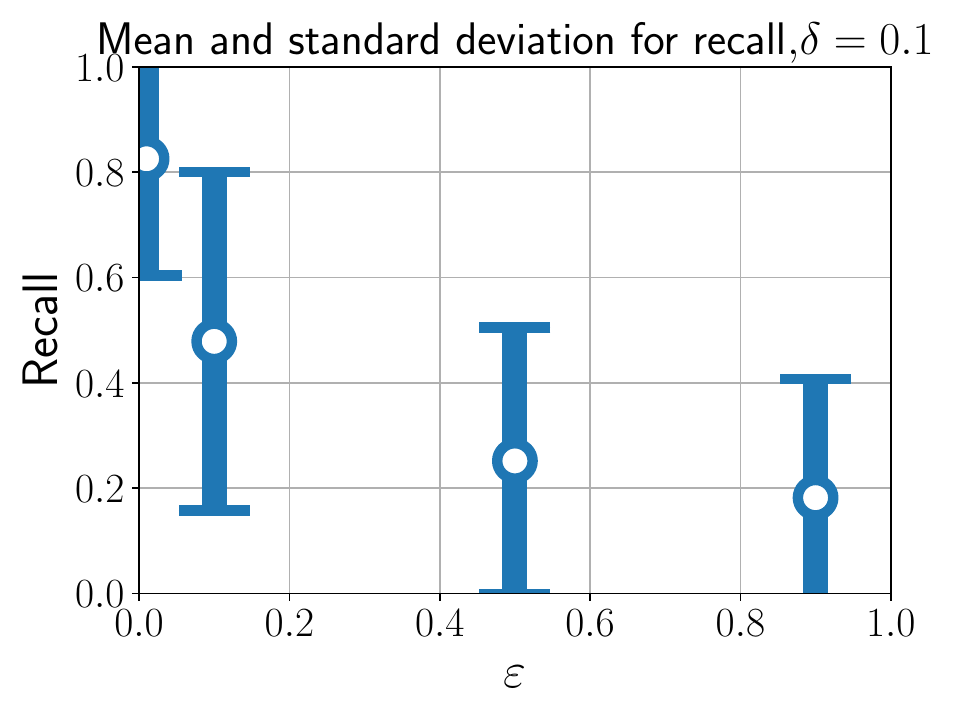}
  \caption{Stability experiment for ZD.  Results with
    $\varepsilon\in\{0.01,0.1,0.5,0.9\}$ for size of implication set (left),
    dataset precision (middle), and recall (right).}
  \label{fig:stab}
\end{figure}

For the mean of the number of implications, as well as for the mean of
the recall, we observe an inverse proportionality for increasing
$\varepsilon$. For the recall, the standard deviation is high in general
and increasing with $\varepsilon$. In contrast, the fraction
of valid implications remains stable for all considered
$\varepsilon$ with only a small increase in the standard deviation.

All plots indicate that the implications obtained by the algorithm
are reliable to a certain degree with respect to multiple runs of the
algorithm. The observed inverse proportionality can be explained by the number
of samples drawn for a fixed $\varepsilon$ being inverse proportional, cf.\@
Algorithm~\ref{algo:approx-equiv}. The high standard deviation for the recall
may be due to the fact that, for larger values of $\varepsilon$, it is more
likely that frequently applicable implications are missing.  However, for
$\varepsilon=0.1$, we obtained on average a recall of 80\%, which is comparably
high.

\paragraph{Example Results Zoo Data}
\label{sec:example-results-zoo}

In Figure~\ref{example:zoo}, we show the set of implications obtained by applying the
PAC attribute exploration algorithm to the Zoo dataset using $\varepsilon=0.01$
and $\delta=0.1$.  Overall, there were 24 implications, 18 of which were valid
in the Zoo dataset. In this case, the approximate precision and recall were
1.00 and 0.92.

\begin{figure}[t]
  \centering
{\footnotesize
  \begin{itemize}
  \item $\{\emph{leg=5}\} \to \{\emph{eggs}, \emph{predator}, \emph{type=7},
\emph{aquatic}\}$
  \item $\{\emph{tail}, \emph{aquatic}\} \to \{\emph{backbone}\}$
  \item $\{\emph{hair}\} \to \{\emph{breathes}\}$
  \item $\{\emph{type=1}\} \to \{\emph{milk}, \emph{backbone},
\emph{breathes}\}$
  \item $\{\emph{type=3}\} \to \{\emph{backbone}, \emph{tail}\}$
  \item $\{\emph{airborne}\} \to \{\emph{breathes}\}$
  \item $\{\emph{type=2}\} \to \{\emph{eggs}, \emph{feathers}, \emph{catsize},
\emph{leg=2}, \emph{backbone}, \emph{tail}, \emph{breathes}\}$ [FALSE]
  \item $\{\emph{type=4}\} \to \{\emph{eggs}, \emph{toothed}, \emph{fins},
\emph{leg=0}, \emph{backbone}, \emph{tail}, \emph{aquatic}\}$
  \item $\{\emph{milk}\} \to \{\emph{type=1}, \emph{backbone},
\emph{breathes}\}$
  \item $\{\emph{leg=6}\} \to \{\emph{eggs}, \emph{type=6}, \emph{airborne},
\emph{breathes}, \emph{hair}, \emph{venomous}\}$ [FALSE]
  \item $\{\emph{domestic}, \emph{catsize}\} \to \{\emph{milk}, \emph{predator},
\emph{toothed}, \emph{type=1}, \emph{backbone}, \emph{breathes}, \emph{hair}\}$
[FALSE]
  \item $\{\emph{tail}, \emph{type=7}\} \to \{\emph{predator}, \emph{leg=8},
\emph{breathes}, \emph{venomous}\}$
  \item $\{\emph{leg=0}, \emph{breathes}, \emph{hair}\} \to \{\emph{milk},
\emph{predator}, \emph{toothed}, \emph{catsize}, \emph{fins}, \emph{type=1},
\emph{backbone}, \emph{aquatic}\}$
  \item $\{\emph{toothed}\} \to \{\emph{backbone}\}$
  \item $\{\emph{type=5}\} \to \{\emph{leg=4}, \emph{eggs}, \emph{toothed},
\emph{backbone}, \emph{breathes}, \emph{aquatic}\}$
  \item $\{\emph{eggs}, \emph{catsize}, \emph{backbone}\} \to \{\emph{tail},
\emph{breathes}\}$ [FALSE]
  \item $\{\emph{leg=2}\} \to \{\emph{backbone}, \emph{breathes}\}$
  \item $\{\emph{leg=8}\} \to \{\emph{predator}, \emph{tail}, \emph{type=7},
\emph{breathes}, \emph{venomous}\}$ [FALSE]
  \item $\{\emph{leg=4}, \emph{breathes}\} \to \{\emph{backbone}\}$
  \item $\{\emph{fins}, \emph{backbone}\} \to \{\emph{toothed},
\emph{aquatic}\}$
  \item $\{\emph{feathers}, \emph{breathes}\} \to \{\emph{type=2}, \emph{eggs},
\emph{catsize}, \emph{leg=2}, \emph{backbone}, \emph{tail}\}$ [FALSE]
  \item $\{\emph{type=6}, \emph{backbone}\} \to \bot$
  \item $\{\emph{leg=4}, \emph{leg=2}, \emph{backbone}, \emph{breathes}\} \to
\bot$
  \item $\{\emph{leg=0}, \emph{backbone}, \emph{breathes}\} \to
\{\emph{predator}, \emph{toothed}\}$
\end{itemize}}
\caption{The result of a particular run of the PAC attribute exploration with
  $\varepsilon=0.01$ and $\delta=0.1$. False implications are marked at the end
  by [FALSE].}
\label{example:zoo}
\end{figure}

\section{Conclusion}
\label{sec:conclusion}

In this paper, we have shown that Horn envelopes of arbitrary domains are
PAC-learnable via implication queries, for which the oracle must confirm that an
implication $A \to B$ is valid in the domain or provide a counterexample to
it. We have considered two notions of approximation of Horn envelopes, one much
stronger than the other one, and provided algorithms to compute both.

There are various possible next steps. One aspect is to optimize the algorithm
through more effective usage of implications that the oracle confirms and
counterexamples it provides. Another interesting modification of the algorithm
would be to change the sampling distribution in order to reduce the number of
queries or to better adapt to a domain while preserving the PAC property. Beyond
that, one may think about adapting the algorithm to learn implications that
admit a certain small fraction of counterexamples (i.e., high-confident
association rules). Other possible settings include learning from error-prone
experts or from multiple experts with partial or even conflicting views on the
domain.

An important potential application of the algorithms presented here is
completing description logic knowledge bases. This has been done with standard
attribute exploration \citep{BGSS07}; we plan to consider a similar application
for its PAC versions proposed in this paper.

Finally, one may think about a \emph{time-constraint exploration} version
suitable for situations when the system has only a limited amount of time to
learn implicational knowledge.

%\section*{References}

\bibliography{pacfca.bib}

\end{document}